\documentclass[runningheads,a4paper]{llncs}

\usepackage{enumerate}
\usepackage{amsmath}
\usepackage{amsfonts}
\usepackage{amssymb}
\usepackage{amsbsy}
\usepackage{isomath}
\usepackage{algorithm}
\usepackage{algorithmicx}
\usepackage{algpseudocode}
\usepackage{mathrsfs}
\usepackage{paralist}
\usepackage{epsfig}
\usepackage{subfigure}
\usepackage{makeidx} 
\usepackage{color}
\usepackage{pstricks}
\usepackage{pst-node}
\usepackage{tikz}
\usepackage{multido} 
\usepackage{varioref}
\usepackage{wrapfig} 
\usepackage[numbers]{natbib}
\usepackage{hyperref}
\usepackage{psfrag}

\numberwithin{equation}{section}

\newcommand \E {\mathop{\mbox{\ensuremath{\mathbb{E}}}}\nolimits}

\renewcommand \Pr {\mathop{\mbox{\ensuremath{\mathbb{P}}}}\nolimits}

\newcommand{\cset}[2]{\left\{#1 \,\middle|\, #2 \right\} }
\newcommand{\measN}{\lambda}
\newcommand{\meas}[1]{\measN\left(#1\right)}

\newcommand \Reals {{\mathbb{R}}}

\newcommand \CA {{\mathcal{A}}}

\newcommand \CL {{\mathcal{L}}}

\newcommand \CO {{\mathcal{O}}}

\newcommand \CR {{\mathcal{R}}}
\newcommand \CS {{\mathcal{S}}}
\newcommand \CT {{\mathcal{T}}}

\newcommand \defn {\mathrel{\triangleq}}

\newcommand \argmax{\mathop{\rm arg\,max}}

\newcommand \Actions {\CA}
\newcommand \States {\CS}

\newcommand \MDPs {\ensuremath{\mathcal{M}}}
\newcommand \mdp {\ensuremath{\mu}}
\newcommand \pol {\ensuremath{\pi}}
\newcommand \Pols {\ensuremath{\mathcal{P}}}
\newcommand \bel {\xi}

\newcommand \rmax {r_{\textrm{max}}}

\newcommand\dd{\,\mathrm{d}}

\newcommand \Regret {\CL}
\newcommand \Util {U}

\newcommand \disc {\gamma}

\tikzstyle{place}=[circle,draw=black,draw=blue!50,fill=blue!20,inner sep=0mm, minimum size=6mm]
\tikzstyle{hidden}=[circle,draw=black,draw=blue!50,fill=blue!01,inner sep=0mm, minimum size=6mm]
\tikzstyle{observed}=[circle,draw=black,draw=blue!50,fill=blue!10,inner sep=0mm, minimum size=6mm]
\tikzstyle{transition}=[rectangle,draw=black!50,fill=black!20,thick]

\newcommand \MCBRL {\textsc{MCBRL}}
\newcommand \MSBI {\textsc{MSBI}}
\newcommand \MMBI {\textsc{MMBI}}
\newcommand \Beetle {\textsc{Beetle}}
\newcommand \Pstat {\bar{\Pols}_1}
\newcommand \emdp {\hat{\mdp}_\bel}

\title{Robust Bayesian reinforcement learning through tight lower bounds}

\author{Christos Dimitrakakis\inst{1}}
\institute{EPFL, Lausanne, Switzerland\\
  \email{christos.dimitrakakis@epfl.ch} 
}

\begin{document}
\mainmatter  

\maketitle
\begin{abstract}
  In the Bayesian approach to sequential decision making, exact calculation of the (subjective) utility is intractable. This extends to most special cases of interest, such as reinforcement learning problems.  While utility bounds are known to exist for this problem, so far none of them were particularly tight. In this paper, we show how to efficiently calculate a lower bound, which corresponds to the utility of a near-optimal {\em memoryless} policy for the decision problem, which is generally different from both the Bayes-optimal policy and the policy which is optimal for the expected MDP under the current belief. We then show how these can be applied to obtain robust exploration policies in a Bayesian reinforcement learning setting.
\end{abstract}

\section{Setting}
\label{sec:setting}
We consider decision making problems where an agent is acting in a
(possibly unknown to it) environment. By choosing actions, the agent
changes the state of the environment and in addition obtains scalar
rewards. The agent acts so as to maximise the expectation of the
utility function: $U_t \defn \sum_{k=t}^{T}\gamma^k r_{k}$, where
$\gamma \in [0,1]$ is a discount factor and where the instantaneous
rewards $r_t \in [0, \rmax]$ are drawn from a Markov decision process
(MDP) $\mdp$, defined on a state space $\CS$ and an action space
$\CA$, both equipped with a suitable metric and $\sigma$-algebra, with
a set of transition probability measures $\cset{\CT_\mdp^{s,a}}{s \in
  \CS, a \in \CA}$ on $\CS$ , and a set of reward probability measures
$\cset{\CR_\mdp^{s,a}}{s \in \CS, a \in \CA}$ on $\Reals$, such that:
\begin{align}
  r_t \mid s_t = s, a_t = a
  &
  \sim \CR_\mdp^{s,a},
  &
  s_{t+1} \mid s_t = s, a_t = a \sim \CT_\mdp^{s,a},
\end{align}
where $s_t \in \States$ and $a_t \in \Actions$ are the state of the
MDP, and the action taken by the agent at time $t$, respectively.  The
environment is controlled via a {\em policy} $\pol \in \Pols$.  This
defines a conditional probability measure on the set of actions, such
that $\Pr_\pol(a_t \in A \mid s^t, a^{t-1}) = \pol(A \mid s^t,
a^{t-1})$ is the probability of the action taken at time $t$ being in
$A$, where we use $\Pr$, with appropriate subscripts, to denote
probabilities of events and $s^t \defn s_1, \ldots, s_t$ and $a^{t-1}
\defn a_1, \ldots, a_{t-1}$ denotes sequences of states and actions
respectively. We use $\Pols_k$ to denote the set of $k$-order Markov
policies. Important special cases are the set of {\em blind} policies
$\Pols_0$ and the set of {\em memoryless} policies $\Pols_1$. A policy
in $\pol \in \bar{\Pols}_k \subset \Pols_k$ is {\em stationary}, when
$\pol(A \mid s_{t-k+1}^t, a_{t-k+1}^{t-1}) = \pol(A \mid s^k, a^{k-1})$ for
all $t$.

The expected utility, conditioned on the policy, states and actions is
used to define a {\em value function} for the MDP $\mdp$ and a
stationary policy $\pol$, at stage $t$:
\begin{align}
  \label{eq:mdp-value-function}
  Q_{\mdp,t}^\pol(s,a) &\defn \E_{\mdp, \pol} (U_t  \mid s_t =s, a_t =a),
  &
  V_{\mdp,t}^\pol(s) &\defn  \E_{\mdp, \pol} (U_t  \mid s_t =s),
\end{align}
where the expectation is taken with respect to the process defined
jointly by $\mdp, \pol$ on the set of all state-action-reward
sequences $(\States, \Actions, \Reals)^*$.  The {\em optimal} value
function is denoted by $Q^*_{\mdp, t} \defn \sup_\pol Q^\pol_{\mdp,
  t}$ and $V^*_{\mdp, t} \defn \sup_\pol V^\pol_{\mdp, t}$.  We denote
the optimal policy\footnote{We assume that there exists at least one
  optimal policy. If there are multiple optimal policies, we choose
  arbitrarily among them.} for $\mdp$ by $\pol_\mdp^*$. Then
$Q^*_{\mdp, t}= Q^{\pol_\mdp^*}_{\mdp, t}$ and $V^*_{\mdp, t}=
V^{\pol_\mdp^*}_{\mdp, t}$.

There are two ways to handle the case when the true MDP is unknown.
The first is to consider a set of MDPs such that the probability of
the true MDP lying outside this set is bounded from
above~\citep[e.g.][]{thesis:Kaelbling,icml:kearns:1998,brafman2003r,JMLR:UCRL2,strehl2009reinforcement,strehl2008analysis}.
The second is to use a Bayesian framework, whereby a full distribution
over possible MDPs is maintained, representing our subjective belief,
such that MDPs which we consider more likely have higher
probability~\citep[e.g.][]{duff2002olc,dearden99bayesian,ICML:Wyatt:2001,Asmuth:BOSS,dimitrakakis-icaart2010}. Hybrid
approaches are relatively rare~\citep{mfard:nips2010}. In this paper,
we derive a method for efficiently calculating near-optimal, robust,
policies in a Bayesian setting.

\subsection{Bayes-optimal policies}
\label{sec:bayesian-formulation}
In the Bayesian setting, our uncertainty about the Markov decision
process (MDP) is formalised as a probability distribution on the class
of allowed MDPs. 
More precisely, assume a probability measure $\bel$ over a set of
possible MDPs $\MDPs$, representing our belief.  The expected utility
of a policy $\pol$ with respect to the belief $\bel$ is:
\begin{equation}
  \label{eq:expected-utility}
  \E_{\bel, \pol} U_t = \int_\MDPs \E_{\mdp,\pol}(U_t) \dd{\bel}(\mdp).
\end{equation}
Without loss of generality, we may assume that all MDPs in $\MDPs$
share the same state and action space. 
For compactness, and with minor abuse of notation, we define the
following value functions with respect to the belief:
\begin{align}
  \label{eq:belief-value-function}
  Q_{\bel,t}^\pol(s,a) &\defn \E_{\bel, \pol} (U_t  \mid s_t =s, a_t =a),
  &
  V_{\bel,t}^\pol(s) &\defn  \E_{\bel, \pol} (U_t  \mid s_t =s),
\end{align}
which represent the expected utility under the belief $\bel$, at stage
$t$, of policy $\pol$, conditioned on the current state and action.
\begin{definition}[Bayes-optimal policy]
  A Bayes-optimal policy $\pol^*_\bel$ with respect to a belief $\bel$
  is a policy maximising (\ref{eq:expected-utility}).  Similarly to
  the known MDP case, we use $Q^*_{\bel,t}, V^*_{\bel,t}$ to denote
  the value functions of the Bayes-optimal policy.
\end{definition}
Finding the Bayes-optimal policy is generally
intractable~\citep{Degroot:OptimalStatisticalDecisions,duff2002olc,Gittins:1989}.
It is important to note that a Bayes-optimal policy is not necessarily
the same as the optimal policy for the true MDP. Rather, it is the
optimal policy given that the true MDP was drawn at the start of the
experiment from the distribution $\bel$. All the theoretical
development in this paper is with respect to $\bel$.

\subsection{Related work and main
  contribution} \label{sec:related-work} Since computation of the
Bayes-optimal policy is intractable in the general case, in this work
we provide a simple algorithm for finding near-optimal {\em
  memoryless} policies in polynomial time.  By definition, for any
belief $\bel$, the expected utility under that belief of any policy
$\pol$ is a lower bound on that of the optimal policy
$\pol^*_\xi$. Consequently, the near-optimal memoryless policy gives
us a tight lower bound on the subjective utility.

A similar idea was used in~\citep{dimitrakakis-icaart2010}, where the
stationary policy that is optimal on the {\em expected MDP} is used to
obtain a lower bound.  This is later refined through a stochastic
branch-and-bound technique that employs a similar upper bound. In a
similar vein, \citep{Furmston:VariationalRL} uses approximate Bayesian
inference to obtain a stationary policy for the current belief.  More
specifically, they consider two families of expectation maximisation
algorithms. The first uses a variational approximation to the
reward-weighted posterior of the transition distribution, while the
second performs expectation propagation on the first two
moments. However, none of the above approaches return the optimal
stationary policy.

It is worthwhile to mention the very interesting point-based \Beetle{}
algorithm of~\citet{poupart2006asd}, which discretised the belief
space by sampling a set of future beliefs (rather than MDPs).  Using
the convexity of the utility with respect to the belief, they
constructed a lower bound via a piecewise-linear approximation of the
complete utility from these samples. The approach results in an
approximation to the optimal non-stationary policy. Although the
algorithm is based on an optimal construction reported in the same
paper, sufficient conditions for its optimality are not known.

In this paper, we obtain a tight lower bound for the {\em current}
belief by calculating a nearly optimal {\em memoryless} policy. The
procedure is computationally efficient, and we show that it results in
a much tighter bound than the value of the expected-MDP-optimal
policy. We also show that it can be used in practice to perform robust
Bayesian exploration in unknown MDPs.  This is achieved by computing a
new memoryless policy once our belief has changed significantly, a
technique also employed by other
approaches~\citep{JMLR:UCRL2,DBLP:conf/nips/AuerJO08,Asmuth:BOSS,strens2000bayesian,ICML:Wyatt:2001}.
It can be seen as a principled generalisation of the sampling approach
suggested in \cite{strens2000bayesian} from a single MDP sample to
multiple samples from the posterior. The crucial difference is that,
while previous work uses some form of {\em optimistic} policy, we
instead employ a more conservative policy in each stationary
interval. This can be significantly better than the policy which is
optimal for the expected MDP.

The first problem we tackle is how to compute this policy given a
belief over a finite number of MDPs. For this, we provide a simple
algorithm based on backwards induction~\cite[see][for
example]{Degroot:OptimalStatisticalDecisions}. In order to extend this
approach to an arbitrary MDP set, we employ Monte Carlo sampling from
the current posterior. Unlike other Bayesian sampling
approaches~\cite{dearden99bayesian,strens2000bayesian,Asmuth:BOSS,castro2010smarter,wang:bayesian-sparse-sampling:icml:2005,dimitrakakis-icaart2010,ICML:Wyatt:2001},
we use these samples to estimate a policy that is nearly optimal
(within the restricted set of memoryless policies) with respect to the
distribution these samples were drawn from. Finally, we provide
theoretical and experimental analyses of the proposed algorithms.

\section{MMBI: Multi-MDP Backwards Induction}
Even when our belief $\bel$ is a probability measure over a finite set
of MDPs $\MDPs$, the finding an optimal policy is intractable. For
that reason, we restrict ourselves to memoryless policues $\pol \in
\Pols_1$.  We can approximate the optimal {\em memoryless} policy with
respect to $\bel$, by setting the posterior measure given knowledge of
the policy so far and the current state, to equal the initial belief,
i.e. $\bel(\mdp \mid s_t = s, \pol) = \bel(\mdp)$ (we do not condition
on the complete history, since the policies are memoryless). The
approximation is in practice quite good, since the difference between
the two measures tends to be small. The policy $\pol_\MMBI$ can then
be obtained via the following backwards induction. By definition:
\begin{align}
  \label{eq:backwards-induction}
  Q_{\bel,t}^\pol(s,a)
  &= \E_{\bel, \pol} (r_t  \mid s_t =s, a_t =a)
  + \gamma \E_{\bel, \pol} (U_{t+1}  \mid s_t =s, a_t =a),
\end{align}
where the expected reward term can be written as
\begin{subequations}
  \label{eq:expected-reward}
  \begin{align}
    \E_{\bel, \pol} (r_t  \mid s_t =s, a_t =a)
    &=
    \int_\MDPs \E_{\mdp} (r_t  \mid s_t =s, a_t =a) \dd{\bel(\mdp)},
    \\
    \E_{\mdp} (r_t  \mid s_t =s, a_t =a)
    &=
    \int_{-\infty}^\infty r \dd{\CR_\mdp^{s,a}}(r).
  \end{align}
\end{subequations}
The next-step utility can be written as:
\begin{subequations}
  \label{eq:next-step-utility}
  \begin{align}
    \E_{\bel, \pol} (U_{t+1}  \mid s_t =s, a_t=a)
    &=
    \int_\MDPs \E_{\mdp, \pol} (U_{t+1}  \mid s_t =s, a_t =a) \dd{\bel}(\mdp),
    \\
    \E_{\mdp, \pol} (U_{t+1}  \mid s_t =s, a_t =a)
    &=
    \int_\CS V_{\mdp,t+1}^\pol(s') \dd{\CT_\mdp^{s,a}}(s').
  \end{align}
\end{subequations}
Putting those steps together, we obtain
Algorithm~\ref{alg:multi-mdp-value-iteration}, which greedily
calculates a memoryless policy for a $T$-stage problem and returns its
expected utility.
\begin{algorithm}
  \begin{algorithmic}[1]
    \Procedure{MMBI}{$\MDPs, \bel, \gamma, T$}
    \State Set  $V_{\mdp,T+1}(s) = 0$ for all $s \in \CS$.
    \For{$t = T, T-1, \ldots, 0$}
    \For{$s \in \CS, a \in \CA$}
    \State Calculate $Q_{\bel,t}(s,a)$ from \eqref{eq:backwards-induction}
    using $\{V_{\mdp,t+1}\}$ \label{step:combine}.
    \EndFor
    \For{$s \in \CS$}
    \State $a^*_{\bel,t}(s) = \argmax \cset{Q_{\bel, t}(s,a)}{a \in \CA}$.
    \For{$\mdp \in \MDPs$}
    \State  $V_{\mdp, t}(s) = Q_{\mdp, t}(s, a^*_{\bel,t}(s))$. \label{step:local}
    \EndFor
    \EndFor
    \EndFor
    \EndProcedure
  \end{algorithmic}
  \caption{MMBI - Backwards induction on multiple MDPs.}
  \label{alg:multi-mdp-value-iteration}
\end{algorithm}
\begin{wrapfigure}{r}{0.5\textwidth}
  \centering
    \psfrag{mean}[r][r][1.0][0]{$\E_{\bel}V^{\pol^*(\emdp)}_\mdp$}
    \psfrag{MMBI}[r][r][1.0][0]{$\max_{\pol \in \Pols_1} \E_\bel V^\pol_\mdp$}
    \psfrag{UB}[r][r][1.0][0]{$\E_{\bel} \max_\pol V^\pol_\mdp$}
    \psfrag{utility}[B][B][1.0][0]{expected utility over all states}
    \psfrag{belief}[B][B][1.0][0]{Uncertain $\Leftarrow$ $\xi$ $\Rightarrow$ Certain}

  \includegraphics[width=0.45\textwidth]{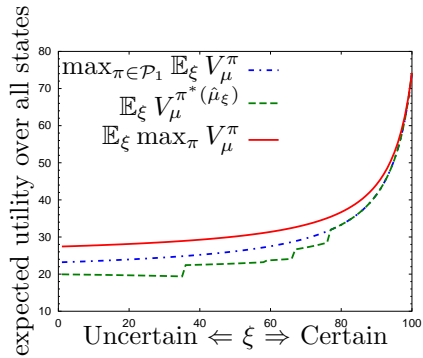}
  \caption{Value function bounds.}
  \label{fig:value-function-bound}
\end{wrapfigure}
The calculation is greedy, since optimising over $\pol$ implies that
at any step $t + k$, we must condition the belief on past policy steps
$\bel(\mdp \mid s_{t+k} = s, \pol_t, \ldots, \pol_{t+k-1})$ to
calculate the expected utility correctly. Thus, the optimal
$\pol_{t+k}$ depends on both future and past selections. Nevertheless,
it is easy to see that Alg.~\ref{alg:multi-mdp-value-iteration}
returns the correct expected utility for time step $t$.
Theorem~\ref{the:optimal-value} bounds the gap between this and the
Bayes-optimal value function when the difference between the current and future beliefs is small.
\begin{theorem}
  For any $k \in [t, T]$, let $\bel_k \defn \bel( \cdot \mid s^k,
  a^k)$ be the posterior after $k$ observations.  Let $\measN$ be a
  dominating measure on $\MDPs$ and $\|f\|_{\measN,1} = \int_\MDPs
  |f(\mdp)| \dd{\meas{\mdp}}$, for any $\measN$-measurable function
  $f$.  If $\|\bel_t - \bel_k\|_{\measN, 1} \leq \epsilon$, for all
  $k$, then the policy $\pol_{\MMBI}$ found by MMBI is within $\rmax (1
  - \disc)^{-2} \epsilon$ of the Bayes-optimal policy $\pol^*_\bel$.
  \label{the:optimal-value}
\end{theorem}
\begin{proof}
  The error at every stage $k > t$, is bounded as follows:
  \begin{align*}
    |V_{\bel, k}(s) - \E_{\bel}(U_k \mid s^k, a^k)|
    &= 
    \left|\int_\MDPs [\bel_t(\mdp) - \bel_k(\mdp)(s)] V_{\mdp,k}(s) \dd{\meas{\mdp}}\right|
    \\
    &\leq
    \frac{\rmax}{1 - \disc} \int_\MDPs |\bel_t(\mdp) - \bel_k(\mdp)(s)|  \dd{\meas{\mdp}}
    \leq
    \frac{\rmax}{1 - \disc} \epsilon.
  \end{align*}
  The final result is obtained via the geometric series. \qed
\end{proof}
We can similarly bound the gap between the MMBI policy and the
$\bel$-optimal memoryless policy, by bounding $\sup_{k, s, \pol}
\|\bel_t(\cdot) - \bel_t( \cdot \mid s_k = s, \pol)\|_{\measN, 1}$.

The $\bel$-optimal memoryless policy is generally {\em different} from
the policy which is optimal with respect to the expected MDP $\emdp
\defn \E_\bel \mdp$, as can be seen via counterexample where $\E_\bel
V^\pol_\mdp \neq V^\pol_{\emdp}$, or even where $\E_\bel \mdp \notin
\MDPs$. MMBI can be used to obtain a much tighter value function bound
than the $\emdp$-optimal policy, as shown in
Fig.~\ref{fig:value-function-bound}, where the MMBI bound is compared
to the $\emdp$-optimal policy bound and the simple upper bound:
$V^*_\bel(s) \leq \E_{\bel} \max_\pol V^\pol_\mdp(s)$. The figure
shows how the bounds change as our belief over 8 MDPs changes. When we
are more uncertain, MMBI is much tighter than
$\emdp$-optimal. However, when most of the probability mass is around
a single MDP, both lower bounds coincide. In further experiments on
online reinforcement learning, described in
Sec.~\ref{sec:experiments}, near-optimal memoryless policies are
compared against the $\emdp$-optimal policy.


\subsection{Computational complexity}
\label{sec:complexity}
When $\MDPs$ is finite and $T < \infty$, \MMBI{}
(Alg.~\ref{alg:multi-mdp-value-iteration}) returns a
greedily-optimised policy $\pol_\MMBI$ and its value function. When $T
\to \infty$, \MMBI{} can be used to calculate an $\epsilon$-optimal
approximation by truncating the horizon, as shown below.
\begin{lemma}
  The complexity of
  Alg.~\ref{alg:multi-mdp-value-iteration} 
  for bounding the value function error by $\epsilon$, is
  $\CO\left(
    \left[
    |\MDPs||\States|^2(|\Actions| + 1)
    + 
    (1 + |\MDPs|) |\States| |\Actions|
    \right]
    \log_\gamma \frac{\epsilon (1 - \disc)}{\rmax}
  \right)$,
  assuming $r_t \in [0,\rmax]$, 
  \label{lem:mmbi-complexity}
\end{lemma}
\begin{proof}
  Since $r_t \in [0, \rmax]$, if we look up to some horizon $T$, our
  value function error is bounded by $\disc^T c$, where $c = H \rmax$
  and $H = \frac{1}{1 - \disc}$ is the effective horizon.
  Consequently, we need $T \geq \log_\disc (\epsilon / c)$ to bound the
  error by $\epsilon$. For each $t$, step~\ref{step:combine} is
  performed $|\States||\Actions|$ times. Each step takes $O(|\MDPs|)$
  operations for the expected reward and $O(|\States||\MDPs|)$
  operations for the next-step expected utility. The second loop is
  $\CO(|\States|(|\Actions| + |\MDPs||\States|))$, since it is
  performed $|\States|$ times, with the max operators taking
  $|\Actions|$ operations, while inner loop is performed $|\MDPs|$
  times with each local MDP update step~\ref{step:local} takes
  $|\States|$ operations.
  \qed
\end{proof}
\begin{algorithm}
  \begin{algorithmic}[1]
    \Procedure{MSBI}{$\bel, \gamma, \epsilon$}
    \State $n = \left(\frac{3\rmax}{\epsilon(1-\disc)}\right)^3$.
    \State $\MDPs = \{\mdp_1, \ldots,
  \mdp_n\}$, $\mdp_i \sim \bel$.
  \State \textsc{MMBI}($\MDPs, p, \disc, \log_\disc \frac{\epsilon(1-\disc)}{\rmax}$), with $p(\mdp_i) = 1/n$ for all $i$.  
    \EndProcedure
  \end{algorithmic}
  \caption{MSBI: Multi-Sample Backwards Induction}
  \label{alg:msbi}
\end{algorithm}
It is easy to see that the most significant term is
$\CO(|\MDPs||\States|^2|\Actions|)$, so the algorithmic complexity
scales linearly with the number of MDPs. Consequently, when $\MDPs$ is
not finite, exact computation is not possible.  However, we can use
high probability bounds to bound the expected loss of a policy
calculated stochastically through \MSBI{} (Alg.\ref{alg:msbi}).

\MSBI{} simply takes a sufficient number of samples of MDPs from
$\bel$, so that in $\bel$-expectation, the loss relative to the
\MMBI{} policy is bounded according to the following lemma.
\begin{lemma}
  The expected loss of \MSBI{} relative to \MMBI{} is bounded by $\epsilon$.
  \label{lem:mc-bound}
\end{lemma}
\begin{proof}
  Let $\hat{\E}^n \Util = \frac{1}{n} \sum_{i=1}^n \E_{\mdp_i} \Util$
  denote the empirical expected utility over the sample of $n$ MDPs,
  where the policy subscript $\pol$ is omitted for simplicity.  Since
  $\E_\bel \hat{\E}^n \Util = \E_\bel \Util$, we can use the Hoeffding
  inequality to obtain:
  \[
  \bel\left(\cset{\mdp^n}{\hat{\E}^n \Util \geq \E_\bel \Util + \epsilon}\right)
  \leq
  e^{-2n\epsilon^2/c^2}.
  \]
This implies the following bound:
  \begin{align*}
    \E_\bel (\hat{\E}^n \Util - \E_\bel \Util) &\leq c \delta + c
    \sqrt{\frac{\ln (1/\delta)}{2n}}
    \leq c (8n)^{-1/3} + c \sqrt{\frac{(8n)^{1/3}}{2n}}
    = 3cn^{-1/3}.
  \end{align*}
  Let $\Pols_1$ be the set of memoryless policies.  Since the bound
  holds uniformly (for any $\pol \in \Pols$), the policy $\hat{\pol}^*
  \in \Pstat$ maximising $\hat{\E}^n$ is within $3cn^{-1/3}$ of the
  $\bel$-optimal policy in $\Pols_1$.  \qed
\end{proof}
Finally, we can combine the above results to bound the complexity of
achieving a small approximation error for \MSBI{}, with respect to expected loss:
\begin{theorem}
  \MSBI{} (Alg.~\ref{alg:msbi}) requires 
  $\CO\left(
    \left(\frac{6 \rmax}{\epsilon(1 - \disc)}\right)^3
    |\States|^2|\Actions|
    \log_\disc \frac{\epsilon(1-\disc)}{2\rmax}
    \right)$
  operations to be $\epsilon$-close to the best MMBI policy.
\end{theorem}
\begin{proof}
  From Lem.~\ref{lem:mc-bound}, we can set $n = (6c/\epsilon)^3$ to
  bound the regret by $\epsilon / 2$. Using the same value in
  Lem.~\ref{lem:mmbi-complexity}, and setting $|\MDPs| = n$, we obtain
  the required result.
  \qed
\end{proof}

\subsection{Application to robust Bayesian reinforcement learning}
While \MSBI{} can be used to obtain a memoryless policy which is in
expectation close to both the optimal memoryless policy and the
Bayes-optimal policy for a given belief, the question is how to extend
the procedure to online reinforcement learning. The simplest possible
approach is to simply recalculate the stationary policy after some
interval $B > 0$. This is the approach followed by \MCBRL{}
(Alg.~\ref{alg:mcbrl}), shown below.
\label{sec:application-bayesian-rl}
\begin{algorithm}
  \begin{algorithmic}[1]
    \Procedure{MCBRL}{$\bel_0, \gamma, \epsilon, B$}
    \State Calculate $\bel_t(\cdot) = \bel_0(\cdot \mid s^t, a^{t-1})$.
    \State Call \MSBI{}($\bel_t, \gamma, \epsilon)$ and run returned policy for $B$ steps.
    \EndProcedure
  \end{algorithmic}
  \caption{MCBRL: Monte-Carlo Bayesian Reinforcement Learning}
  \label{alg:mcbrl}
\end{algorithm}

\section{Experiments in reinforcement learning problems}
\label{sec:experiments}
Selecting the number of samples $n$ according to $\epsilon$ for
\MCBRL{} is computationally prohibitive.  In practice, instead of
setting $n$ via $\epsilon$, we simply consider increasing values of
$n$. For a single sample ($n=1$), \MCBRL{} is equivalent to the
sampling method in~\cite{strens2000bayesian}, which at every new
stage, samples a single MDP from the current posterior and then uses
the policy that is optimal for the sampled MDP.  In addition, for this
particular experiment, rather than using the memoryless policy found,
we apply the stationary policy derived by using the first step of the
memoryless policy. This incurs a small additional loss.
\begin{algorithm}[htb]
  \begin{algorithmic}[1]
    \Procedure{Exploit}{$\bel_0, \gamma$}
    \For {$t=1, \ldots}$
    \State Calculate $\bel_t(\cdot) = \bel_0(\cdot \mid s^t, a^{t-1})$.
    \State Estimate $\hat{\mdp}_{\bel_t} \defn \E_{\bel_t} \mdp$.
    \State Calculate $Q^*_{\hat{\mdp}_{\bel_t}}(s,a)$ using discount parameter $\gamma$.
    \State Select $a_t = \argmax_a Q^*_{\hat{\mdp}_{\bel_t}}(s,a)$
    \EndFor
    \EndProcedure
  \end{algorithmic}
  \caption{\textsc{Exploit}: Expected MDP exploitation~\cite{poupart2006asd}}
  \label{alg:expected}
\end{algorithm}
We also compared \MCBRL{} against the common heuristic of acting
according to the policy that is optimal with respect to the {\em
  expected} MDP $\emdp \defn \E_\bel \mdp$.  The algorithm, referred
to as the \textsc{Exploit} heuristic in \cite{poupart2006asd}, is
shown in detail in Alg.~\ref{alg:expected}. At every step, this
calculates the expected MDP by obtaining the expected transition
kernel and reward function under the current belief.  It then acts
according to the optimal policy with respect to $\emdp$.  This policy
may be much worse than the optimal policy, even within the class of
stationary policies $\Pstat$.

\begin{figure}[htb]
  \centering
  \subfigure[Expected regret estimate]{
    \psfrag{nsamples}[B][B][1.0][0]{$n$}
    \psfrag{regret}[B][B][1.0][0]{$\Regret$}
    \psfrag{sampling}[B][B][0.7][0]{Alg.~\ref{alg:mcbrl}}
    \psfrag{expected}[B][B][0.7][0]{Alg.~\ref{alg:expected}}
    \includegraphics[width=0.45\textwidth]{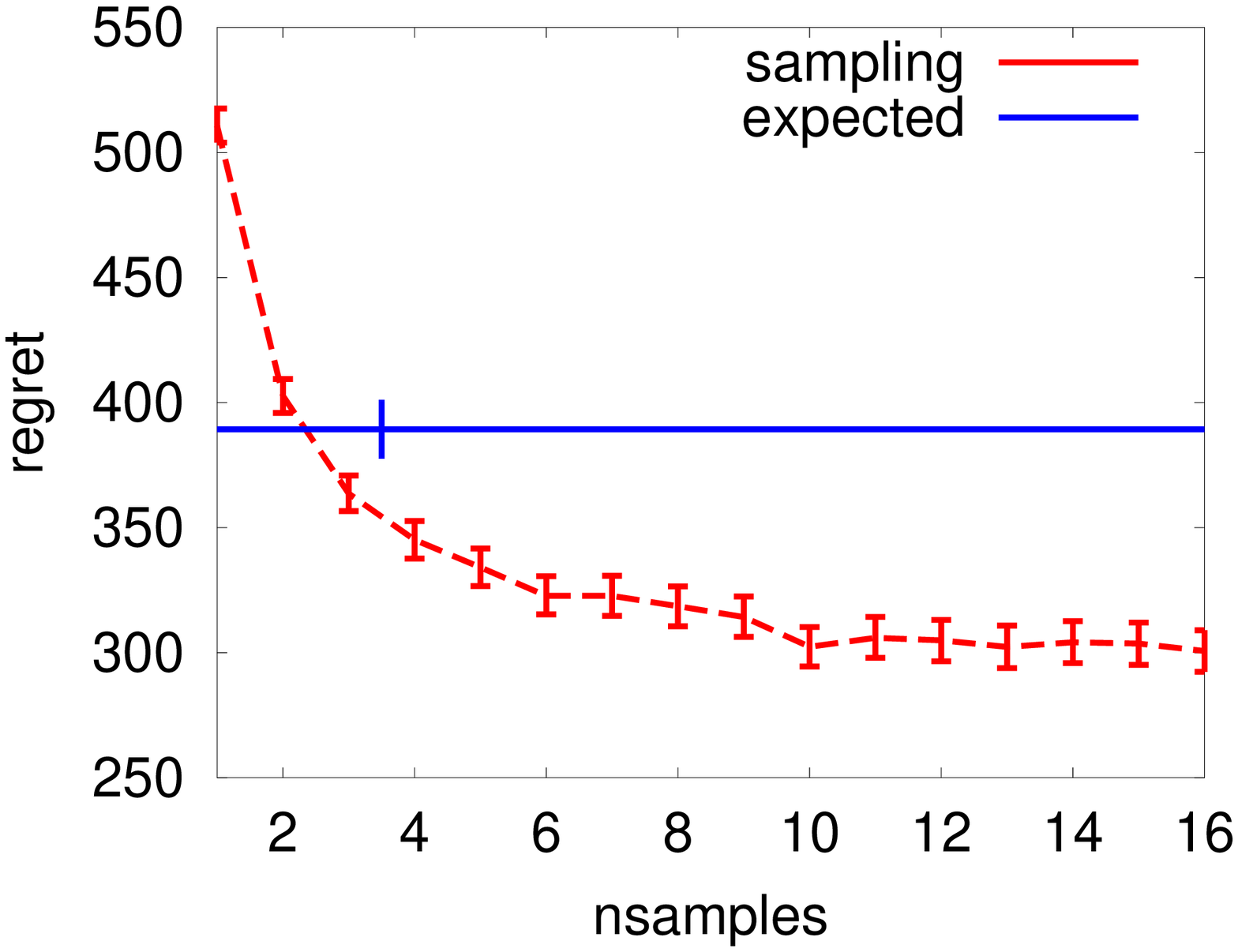}
    \label{fig:chain-task-regret}
  }
  \subfigure[Empirical performance distribution]{
    \psfrag{total reward}[B][B][1.0][0]{total reward $\times 10^{-3}$}
    \psfrag{count}[B][B][1.0][0]{number of runs}
    \psfrag{s1}[B][B][0.7][0]{$n=1$}
    \psfrag{s8}[B][B][0.7][0]{$n=8$}
    \psfrag{model}[B][B][0.7][0]{Alg.~\ref{alg:expected}}
    \includegraphics[width=0.45\textwidth]{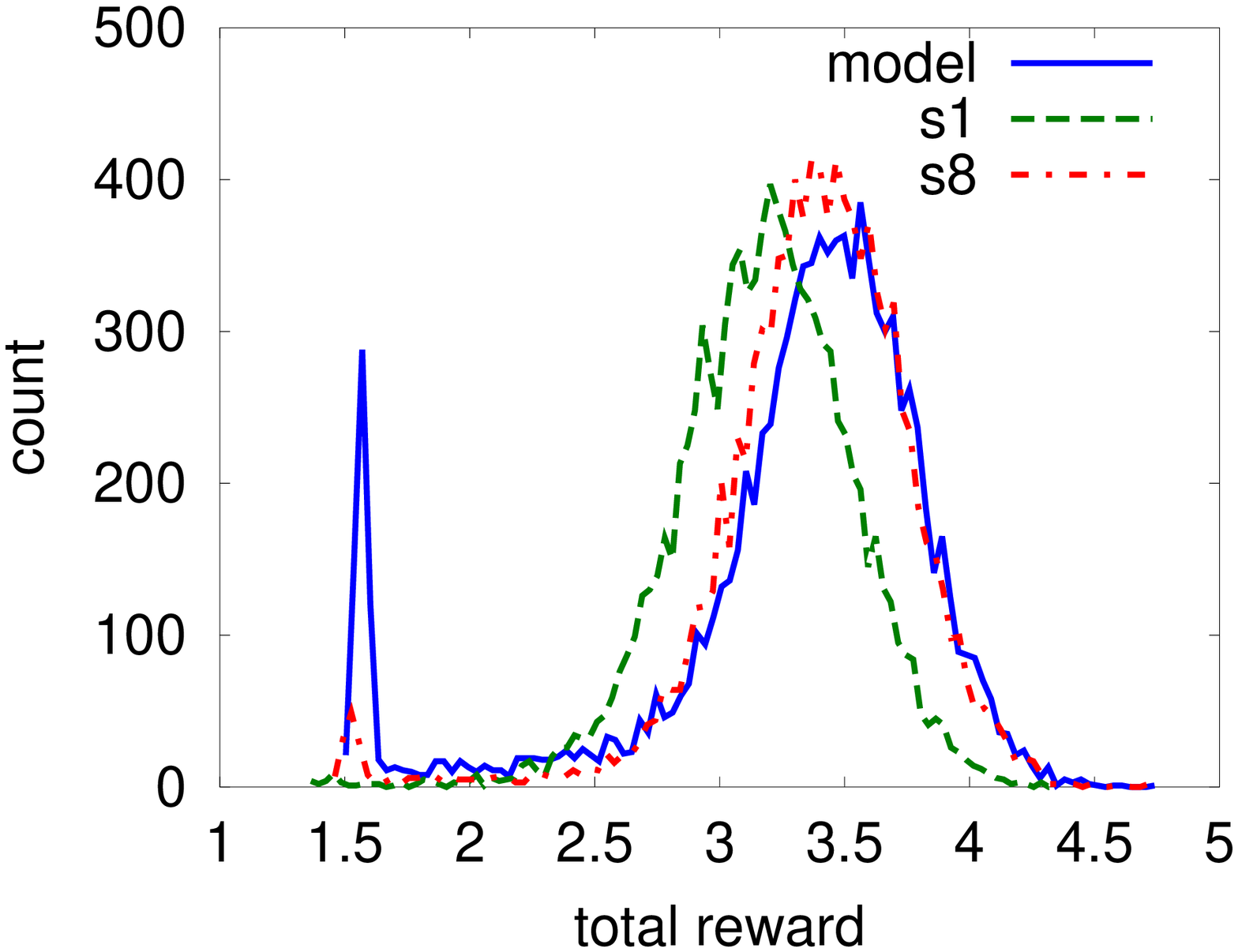}
    \label{fig:chain-task-empirical}
  }
  \caption{Performance on the chain task, for the first $10^3$ steps,
    over $10^4$ runs. \subref{fig:chain-task-regret}: Expected regret
    relative to the optimal (oracle) policy. The {\em sampling} curve
    shows the regret of Alg.~\ref{alg:mcbrl}, as the number of samples
    increases, with $95\%$ confidence interval calculated via a
    $10^4$-boostrap. The {\em expected} curve shows the performance of
    an algorithm acting greedily with respect to the expected MDP.
    \subref{fig:chain-task-empirical}: Empirical distribution of total
    rewards for: the {\em expected} MDP approach and \MCBRL{} with
    $n=1$ and $n=8$ samples.}
  \label{fig:chain-task}
\end{figure}
We compared the algorithms on the Chain task~\cite{dearden98bayesian},
commonly used to evaluate exploration in reinforcement learning
problems. Traditionally, the task has a horizon of $10^3$ steps, a
discount factor $\gamma = 0.95$, and the expected total reward
$\E_{\mdp, \pol} \sum_{t=1}^T r_t$ is compared. We also report the
expected utility $\E_{\mdp, \pol} U_t$, which depends on the discount
factor. All quantities are estimated over $10^4$ runs with
appropriately seeded random number generators to reduce
variance.\footnote{In both cases this expectation is with respect to
  the distribution induced by the actual MDP $\mdp$ and policy $\pol$
  followed, rather than with respect to the belief $\bel$.} The
initial belief about the state transition distribution was set to be a
product-Dirichlet prior
\cite[see][]{Degroot:OptimalStatisticalDecisions} with all parameters
equal to $|\CS|^{-1}$, while a product-Beta prior with parameters
$(1,1)$ was used for the rewards.

Figure~\ref{fig:chain-task} summarises the results in terms of total
reward. The left hand side~(\ref{fig:chain-task-regret}) shows the
expected difference in total reward between the optimal policy
$\pol^*$ and the used policy $\pol$, over $T$ steps, i.e. the regret:
$\Regret = \E_{\mdp, \pol} \sum_{t=1}^T r_t - \E_{\mdp, \pol}
\sum_{t=1}^T r_t$.  The error bars denote $95\%$ confidence intervals
obtained via a $10^4$-bootstrap~\cite{Efron:1993:Introduction}. For
$n=1$, \MCBRL{} performs worse than the expected MDP approach, in
terms of {\em total reward}. On the other hand, as the number of
samples increase, its performance monotonically improves.

Some more detail on the behaviour of the algorithms is given in
Figure~\ref{fig:chain-task-empirical}, which shows the empirical
performance distribution in terms of total reward. The expected MDP
approach has a high probability of getting stuck in a sub-optimal
regime. On the contrary, \MCBRL{}, for $n=1$, results in significant
over-exploration of the environment. However, as $n$ increases,
\MCBRL{} explores significantly less, while the number of runs where
we are stuck in the sub-optimal regime remains small ($<1\%$ of the
runs).
\begin{table}
  \centering
  \begin{tabular}{c|c|c|c}
    Model & $\sum_{t=1}^{1000} r_t$ ($\E U$) & 80\% percentile& confidence interval\\
    \hline
    Alg.~\ref{alg:expected} & 3287 (26.64) & 2518 -- 3842 & 3275 -- 3299 \\
    $n=1$ & 3166 (28.50) & 2748 -- 3582 & 3159 -- 3173 \\
    $n=8$ & 3358 (29.65) & 2932 -- 3800 & 3350 -- 3366 \\
    $n=16$ & 3376 (29.95) & 2946 -- 3830 & 3368 -- 3384\\
    \hline
    \hline
    Model & $\sum_{t=1}^{1000} r_t$ & \multicolumn{2}{c}{Standard interval}\\
    \hline
    \Beetle{}~\cite{poupart2006asd} & 1754 & \multicolumn{2}{|c}{1712--1796}  \\
    AMP-EM~\cite{Furmston:VariationalRL} & 2180 & \multicolumn{2}{|c}{2108--2254}\\
    SEM~\cite{Furmston:VariationalRL} & 2052 & \multicolumn{2}{|c}{2000 --2111}
  \end{tabular}
  \caption{Comparative results on the chain task. The 80\% percentile interval is such that no more than 10\% of the runs were above the maximum or below the minimum value. The confidence interval on the accuracy of the mean estimate, is the 95\% bootstrap interval. The results
    for \Beetle{} and the EM algorithms were obtained from the cited papers, with and the interval based on the reported standard deviation.}
  \label{tab:chain}
\end{table}
Table~\ref{tab:chain} presents comparative results on the chain task
for Alg.~\ref{alg:expected} and for \MCBRL{} for $n \in \{1, 8, 16\}$
in terms of the total reward received in $10^3$ steps. This enables us
to compare against the results reported in
\cite{poupart2006asd,Furmston:VariationalRL}. While the performance of
Alg.~\ref{alg:expected} may seem surprisingly good, it is actually in
line with the results reported in~\cite{poupart2006asd}. Therein,
\Beetle{} only outperformed Alg.~\ref{alg:expected} in the {\em Chain}
task when stronger priors were used. In addition, we would like to
note that while the case $n=1$ is worse than Alg.~\ref{alg:expected}
for the total reward metric, this no longer holds when we examine the
expected utility, where an improvement can already be seen for $n=1$.

\section{Discussion}
We introduced \MMBI{}, a simple backwards induction procedure, to
obtain a near-optimal memoryless policy with respect to a belief over
a finite number of MDPs. This was generalised to \MSBI{}, a stochastic
procedure, whose loss is close in expectation to \MMBI{}, with a gap
that depends polynomially on the number of samples, for a belief on
arbitrary set of MDPs. It is shown that \MMBI{} results in a much
tighter lower bound on the value function that the value of the
$\emdp$-optimal policy. In addition, we prove a bound on the gap
between the value of the $\MMBI{}$ policy and the Bayes-optimal
policy. Our results are then applied to reinforcement learning
problems, by using the \MCBRL{} algorithm to sample a number of MDPs
at regular intervals. This can be seen as a principled generalisation
of \cite{strens2000bayesian}, which only draws one sample at each such
interval. Then \MSBI{} is used to calculate a near-optimal memoryless
policy within each interval. We show experimentally that this
performs significantly better than following the $\emdp$-optimal
policy. It is also shown that the performance increases as we make the
bound tighter by increasing the number of samples taken.

Compared to results reported for other Bayesian reinforcement learning
approaches on the {\em Chain} task, this rather simple method performs
surprisingly well. This can be attributed to the fact that at each
stage, the algorithm selects actions according to a nearly-optimal
stationary policy.

In addition, \MSBI{} itself could be particularly useful for {\em
  inverse} reinforcement learning problems (see for example
\cite{abbeel2004apprenticeship,Ng00algorithmsfor}) where the
underlying dynamics are unknown, or to multi-task problems\cite{Snel:Multitask}. Then it would be possible to obtain
good stationary policies that take into account the uncertainty over
the dynamics, which should be better than using the expected MDP
heuristic. More specifically, in future work, MMBI will be used to
generalise the Bayesian methods developed
in~\cite{dimitrakakis:bmirl:ewrl:2011,rothkopf:peirl:ecml:2011} for
the case of unknown dynamics.

In terms of direct application to reinforcement learning, \MSBI{}
could be used in the inner loop of some more sophisticated method than
\MCBRL{}. For example, it could be employed to obtain tight lower
bounds for the leaf nodes of a planning tree such
as\cite{dimitrakakis-icaart2010}.  By tight integration with such
methods, we hope to obtain improved performance, since we would be
considering wider policy classes.  In a related direction, it would be
interesting to see examine better upper
bounds~\citep{deFarias:CS-LP-ADP:2004,deFarias:LP-ADP:2003,Rogers:Pathwise-Control}
and in particular whether the information relaxations discussed by
\citet{Brown:Information-Duality-SDP} could be extended to the 
Bayes-optimal case.

\section*{Acknowledgements}
Many thanks to Matthijs Snel and Shimon Whiteson for extensive
discussions on the optimality of the MMBI algorithm, and for helping
to discover an error. In addition, I would like to thank Nikos Vlassis
and the anonymous reviewers for helpful comments. This work was
partially supported by the EU-Project IM-CLeVeR, FP7-ICT-IP-231722,
and the Marie Curie Project ESDEMUU, Grant Number 237816.


\end{document}